\documentclass[sn-mathphys]{sn-jnl}
\usepackage{graphicx}
\usepackage{subcaption}
\usepackage{caption}
\usepackage{float}

\jyear{2021}%

\theoremstyle{thmstyleone}%
\newtheorem{thm}{Theorem}
\newtheorem{prop}[thm]{Proposition}%
\newtheorem{lem}[thm]{Lemma} 

\theoremstyle{thmstyletwo}%

\theoremstyle{thmstylethree}%

\begin{document}

\title[Test Set Sizing for Regularized Models]{Test Set Sizing for Regularized Models}


\author{\fnm{Alexander} \sur{Dubbs}}\email{alex.dubbs@gmail.com}

\affil{\orgaddress{\street{400 Central Park West}, \city{New York}, \postcode{10025}, \state{NY}, \country{United States}}}


\abstract{We derive the ideal train/test split for the ridge regression in the limit that the number of training rows $m$ becomes large. The split must depend on the ridge tuning parameter, $\alpha$, but we find that the dependence can asymptotically be ignored; all parameters vanish except for $m$ and the number of features, $n$, which is held constant. This is the first time that such a split is calculated mathematically for a machine learning model in the large data limit. The goal of the calculations is to maximize ``integrity,'' so that the measured error in the trained model is as close as possible to what it theoretically should be. This paper's result for the ridge regression split matches prior art for the plain vanilla linear regression split to the first two terms asymptotically.}




\maketitle

\section{Introduction}\label{sec1}

The question of how to divide one's data into a training set and a test set has long been of theoretical and practical interest to data scientists, and the question is more difficult with the presence of tuning parameters. This paper finds the train/test split for the ridge regression using a two-term asymptotic formula independent of its tuning parameter, $\alpha$, using the Integrity Metric (IM) introduced for the plain vanilla linear regression by the author in \cite{Dubbs2024}, since this case is mathematically tractable. The IM measures the degree to which the measured model error differs from the true error of the model, and this quantity should always be minimized to gain an honest assessment of a model's performance. We pick the number of points $p$ in the training set to minimize the IM. Note that we do not pick $p$ to maximize the measured model accuracy, since then we would derive an assessment of the model's ability that is not truthful. Our main result is:
\vskip .05in
\noindent {\bf Theorem 6.} {\it Let $X$ be a $m \times n$ matrix of normals with independent rows with covariance $\Sigma$. $b \sim N(0,c^2I)$ and $\epsilon \sim N(0,I)$. $\alpha > 0$, $\sigma > 0$, and $m$ will be assumed to be large with $n$ held constant.
$$ y = Xb + \sigma\epsilon $$
$$ \hat{b} = (X_{1:p,:}^tX_{1:p,:} + \alpha I)^{-1}X_{1:p,:}^ty_{1:p} $$
\begin{equation} {\rm IM} = E\left(\left(\frac{1}{m - p}\cdot\|X_{p+1:m,:}\hat{b} - y_{p+1:m}\|^2 - \sigma^2\right)^2\right) \end{equation}
\noindent The value of $p$ that minimizes the Integrity Metric (IM) is
$$ n^{1/3}(2 + n)^{1/3}\cdot m^{2/3} - \frac{2n^{2/3}(1 + n)}{3(2 + n)^{1/3}}\cdot m^{1/3} + {\rm L.O.T.}$$
and the ratio of training data over testing data goes to $0$ as $m$ becomes large.}

The ridge regularization term was first added to the linear regression's loss function in \cite{Hoerl1970} in a statistical context and has become standard in machine learning. This result agrees with prior art \cite{Dubbs2024} except that $b$ is random instead of fixed, which would not change the result in \cite{Dubbs2024} in the $\alpha = 0$ case, so for that parameter value the results match perfectly as is expected. We present computational evidence of the result's accuracy leaving out the lower order terms in a variety of parameter regimes, although since $n$ is fixed the expression does not converge in all parameter regimes.

This sort of analysis has historical precedence in theoretical machine learning. Let $H$ be the hypothesis class of functions $h$ from $\mathbb{R}^n$ to $\{0,1\}$, let $L_{D}(h)$ be the expected absolute loss of $h$ over the distribution $D$ on $\mathbb{R}^n\times\{0,1\}$, and let $L_{S}(h)$ be the expected absolute ``loss'' or error of $h$ where the $m$ elements of $S$ are chosen by $D$ in an i.i.d. fashion. Let $d$ be the VC-Dimension of $H$. Then with probability $1 - \delta$ by \cite{Shalev2014}, Chapter 6,
$$ \lvert L_{S}(h) - L_{D}(h) \rvert \leq \frac{1}{\delta}\sqrt{\frac{2d\log(2em/d)}{m}} $$
This bound is similar in spirit to this paper's result in that it controls the error in the error estimate, this paper considers the case where $h$ is known to be optimal as trained on $S$ but for a less general model space $H$, and it is concerned with the realistic constraint that the number of known data points with labels is finite.

Furthermore, others have provided answers to the question of training/testing division in statistical models with a fixed given amount of data, and \cite{Dubbs2024} discusses their work in its introduction, but none have ever provided similar answers for any machine learning models in which regularization is explicitly included. \cite{Guyon1997} and \cite{Guyon1998} try to minimize test-set error in learning with discrete labels, and find an $O(m)$ solution for test set size.  \cite{Kearns1997} does the same using a tuning parameter and finds another answer. \cite{Afendras2019}, \cite{Joseph2022}, and \cite{Picard1990} find that the training set should grow as $O(m)$ for different metrics than this paper's for unregularized models with potentially continuous labels.

\section{Preliminaries}

\begin{prop}
Let $X$ be an $p \times n$ matrix of Gaussians with mean zero where the rows have covariance $\Sigma$, and $A$, $B$, and $\Sigma$ are positive definite and commute. Then:
\begin{equation} E\left({\rm tr}\left(\Sigma\left(X^tX\right)^{-1}\right)\right) = \frac{n}{p - n - 1}
\end{equation}
\begin{equation} E\left({\rm tr}\left(\Sigma\left(X^tX\right)^{-1}\right)^2\right) = \frac{n^2}{p^2} + O\left(p^{-3}\right) \end{equation}
\begin{equation}E\left({\rm tr}\left(\left(\Sigma(X^tX)^{-1}\right)^2\right)\right)  = \frac{n}{p^2} + O\left(p^{-3}\right) \end{equation}
\end{prop}
\begin{proof}
\vskip -.2in
(2.1) is well known. (2.2) and (2.3) are consequences of Corollary 3.1 in \cite{Rosen1988}, which is an excellent reference for many related results, as are \cite{Hillier2021} and \cite{Holgersson2020}.
\end{proof}
\begin{lem}
It is well-known that if $S$ and $T$ are positive definite matrices, ${\rm tr}(ST) \leq {\rm tr}(S){\rm tr}(T)$.
\end{lem}
\begin{lem} It is clear by expanding in terms of eigenvalues that if $X^tX$ is an $n \times n$ Wishart matrix with $a$ degrees of freedom, mean zero, and covariance $\Sigma$, and $\alpha > 0$,
$$ {\rm tr}\left((I + \alpha X^tX)^{-1}\right) \leq n, $$
$$ {\rm tr}\left((X^tX + \alpha)^{-1}\right) \leq \frac{n}{\alpha}. $$
\end{lem}
\begin{lem}
Let $X$ be an $p \times n$ matrix of Gaussians with mean zero where the rows have covariance $\Sigma$, and $D$ be a possibly correlated positive definite random matrix with maximum eigenvalue $\lambda$, then as $p$ becomes large,
$$ p^k\cdot E\left({\rm tr}\left((X^tX)^{-1}\right)^{k}{\rm tr}\left(D\right)\right) = O(1)$$
\end{lem}
\begin{proof}
$p\cdot E\left({\rm tr}\left((X^tX)^{-1}\right)\right)$ converges in distribution to a point mass, so the correlations among the terms in the product that make up $p^k\cdot E\left({\rm tr}\left((X^tX)^{-1}\right)^{k}\right)$ do not matter, so it converges to a point mass, and
$$ p^k\cdot E\left({\rm tr}\left((X^tX)^{-1}\right)^{k}{\rm tr}\left(D\right)\right) \leq np^k\lambda \cdot E\left({\rm tr}\left((X^tX)^{-1}\right)^{k}\right), $$
so the lemma follows.
\end{proof}

\begin{lem}
Let $X$ be an $p \times n$ matrix of Gaussians with mean zero where the rows have covariance $\Sigma$, we prove that
\begin{equation} E\left({\rm tr}\left(\Sigma \left(X^tX + \alpha I\right)^{-2}\right)\right) = O\left(\frac{1}{p^2}\right) \end{equation}
\begin{equation} E\left({\rm tr}\left(\left(\Sigma \left(X^tX + \alpha I\right)^{-2}\right)^2\right)\right) = O\left(\frac{1}{p^4}\right) \end{equation}
\begin{equation} E\left({\rm tr}\left(\Sigma \left(X^tX + \alpha I\right)^{-2}\right)^2\right) = O\left(\frac{1}{p^4}\right) \end{equation}
\begin{equation} p\cdot E\left({\rm tr}\left(\Sigma X^tX\left(X^tX + \alpha I\right)^{-2} \right)\right) = n + O\left(\frac{1}{p}\right) \end{equation}
\begin{equation} p^2\cdot E\left({\rm tr}\left(\left(\Sigma X^tX\left(X^tX + \alpha I\right)^{-2}\right)^2\right)\right) = n +O\left(\frac{1}{p}\right)\end{equation}
\begin{equation} p^2\cdot E\left({\rm tr}\left(\Sigma X^tX\left(X^tX + \alpha I\right)^{-2}\right)^2\right) = n^2  + O\left(\frac{1}{p}\right)  \end{equation}
\begin{equation} p\cdot E\left({\rm tr}\left(\Sigma \left(X^tX + \alpha I\right)^{-2}\Sigma X^tX\left(X^tX + \alpha I\right)^{-2}\right)\right) = O\left(\frac{1}{p^2}\right) \end{equation}
\begin{equation}\hspace*{-.2in} p\cdot E\left({\rm tr}\left(\Sigma \left(X^tX + \alpha I\right)^{-2}\right){\rm tr}\left(\Sigma X^tX\left(X^tX + \alpha I\right)^{-2}\right)\right) = O\left(\frac{1}{p^2}\right) \end{equation}

\vspace*{.05in}

\vspace*{-.3in}
\end{lem}
\vspace*{.1in}
\begin{proof}
All expressions are evaluated using the geometric series formula, Proposition 1, and Lemmas 2,3, and 4.\vspace*{.1in}
For (5),
\begin{align*}
    E\left(\operatorname{tr}\left(\Sigma \left(X^tX + \alpha I\right)^{-2}\right)\right) 
    & \leq E\left(\operatorname{tr}\left(\Sigma\right)\operatorname{tr}\left(\left(X^tX + \alpha I\right)^{-2}\right)\right) \\
    & \leq E\left(\operatorname{tr}\left(\Sigma\right)\operatorname{tr}\left((X^tX)^{-1}\right)^2\right) = O\left(p^{-2}\right)
\end{align*}
For (6),\vspace*{.1in}
\begin{align*}
   E\left(\operatorname{tr}\left(\left(\Sigma \left(X^tX + \alpha I\right)^{-2}\right)^2\right)\right)
    & \leq E\left(\operatorname{tr}\left(\Sigma\right)^2\operatorname{tr}\left(\left(X^tX + \alpha I\right)^{-1}\right)^4\right) \\
    & \leq E\left(\operatorname{tr}\left(\Sigma\right)^2\operatorname{tr}\left((X^tX)^{-1}\right)^4\right) = O\left(p^{-4}\right)
\end{align*}
For (7),\vspace*{.1in}
\begin{align*}
    E\left(\operatorname{tr}\left(\left(\Sigma \left(X^tX + \alpha I\right)^{-2}\right)\right)^2\right)
    & \leq E\left(\operatorname{tr}\left(\Sigma\right)^2\operatorname{tr}\left(\left(X^tX + \alpha I\right)^{-1}\right)^4\right) \\
    & \leq E\left(\operatorname{tr}\left(\Sigma\right)^2\operatorname{tr}\left((X^tX)^{-1}\right)^4\right) = O\left(p^{-4}\right)
\end{align*}
For (8),\vspace*{.1in}
$$
   p\cdot E\left(\operatorname{tr}\left(\Sigma X^tX\left(X^tX + \alpha I\right)^{-2} \right)\right) 
    $$$$= p\cdot E\left(\operatorname{tr}\left(\Sigma (X^tX)^{-1}\left(I - \alpha (X^tX)^{-1}\left(I + \alpha (X^tX)^{-1}\right)\right)^{2} \right)\right) $$$$= \frac{np}{p - n - 1} + O\left(p^{-1}\right)$$
For (9),\vspace*{.1in}
$$
   \hspace*{-.7in}p^2\cdot E\left(\operatorname{tr}\left(\left(\Sigma X^tX\left(X^tX + \alpha I\right)^{-2}\right)^2\right)\right) 
    $$$$ \hspace*{.59in}= p^2\cdot E\left(\operatorname{tr}\left(\left(\Sigma (X^tX)^{-1}\left(I + \alpha (X^tX)^{-1}\right)^{-2}\right)^2\right)\right)
    $$$$ = p^2\cdot E\left(\operatorname{tr}\left(\left(\Sigma (X^tX)^{-1}\right)^2\right)\right) + O\left(p^{-1}\right)
    $$$$ = n + O\left(p^{-1}\right)
$$
For (10),\vspace*{.1in}
$$
    \hspace*{-.7in}p^2\cdot E\left(\operatorname{tr}\left(\Sigma X^tX\left(X^tX + \alpha I\right)^{-2}\right)^2\right) 
    $$$$ \hspace*{.55in}= p^2\cdot E\left(\operatorname{tr}\left(\Sigma (X^tX)^{-1}\left(I + \alpha (X^tX)^{-1}\right)^{-2}\right)^2\right)
    $$$$ = p^2\cdot E\left(\operatorname{tr}\left(\Sigma (X^tX)^{-1}\right)^2\right) + O\left(p^{-1}\right)
    $$$$ \hspace*{.19in}= n^2  + O\left(p^{-1}\right)
$$
For (11),\vspace*{.1in}
$$
    \hspace*{-1in}p\cdot E\left(\operatorname{tr}\left(\Sigma \left(X^tX + \alpha I\right)^{-2}\Sigma X^tX\left(X^tX + \alpha I\right)^{-2}\right)\right) 
    $$$$ \hspace*{.5in}= p\cdot E\left(\operatorname{tr}\left(\Sigma (X^tX)^{-2}\left(I + \alpha (X^tX)^{-1}\right)^{-2}\Sigma (X^tX)^{-1}\left(I + \alpha (X^tX)^{-1}\right)^{-2}\right)\right) 
$$$$
    \hspace*{-1.85in}\leq p\cdot E\left(\operatorname{tr}\left(\Sigma\right)^2\operatorname{tr}\left((X^tX)^{-1}\right)^3\right)
    $$$$
   \hspace*{-3.1in}= O\left(p^{-2}\right)
$$
For (12),\vspace*{.1in}
$$
    \hspace*{-1in}p\cdot E\left(\operatorname{tr}\left(\Sigma \left(X^tX + \alpha I\right)^{-2}\right)\operatorname{tr}\left(\Sigma X^tX\left(X^tX + \alpha I\right)^{-2}\right)\right) 
    $$$$ \hspace*{-1.1in}\leq p\cdot E\left({\rm tr}(\Sigma)^2{\rm tr}\left((X^tX)^{-1}\right){\rm tr}\left((X^tX)^{-2}\right)\right) \\
    $$$$ \hspace*{-3.05in}= O\left(p^{-2}\right)
$$
\end{proof}

\section{Main Result}

\begin{thm}
Let $X$ be a $m \times n$ matrix of normals with independent rows with covariance $\Sigma$. $b \sim N(0,c^2I)$ and $\epsilon \sim N(0,I)$. $\alpha > 0$, $\sigma > 0$, and $m$ will be assumed to be large.
$$ y = Xb + \sigma\epsilon $$
$$ \hat{b} = (X_{1:p,:}^tX_{1:p,:} + \alpha I)^{-1}X_{1:p,:}^ty_{1:p} $$
\begin{equation} {\rm IM} = E\left(\left(\frac{1}{m - p}\cdot\|X_{p+1:m,:}\hat{b} - y_{p+1:m}\|^2 - \sigma^2\right)^2\right) \end{equation}
\vskip .05in
\noindent The value of $p$ that minimizes the Integrity Metric (IM) is
$$ n^{1/3}(2 + n)^{1/3}\cdot m^{2/3} - \frac{2n^{2/3}(1 + n)}{3(2 + n)^{1/3}}\cdot m^{1/3} + {\rm L.O.T.} $$
and the ratio of training data over testing data goes to $0$ as $m$ becomes large.
\end{thm}
\vspace*{-.3in}
\begin{proof}
Consider
$$ \left(\frac{1}{m - p}\cdot\|X_{p+1:m,:}\hat{b} - y_{p+1:m}\|^2 - \sigma^2\right)^2 $$$$  \hspace*{-.5in} = \Bigg( \frac{1}{m-p}\cdot\Big\|X_{p+1:m,:}(X_{1:p,:}^tX_{1:p,:} +\alpha I)^{-1}X_{1:p,:}^t(X_{1:p,:}b+\sigma\epsilon_{1:p}) $$$$  \hspace*{-.1in} - X_{p+1:m,:}b - \sigma\epsilon_{p+1:m}\Big\|^2 - \sigma^2 \Bigg)^2 $$

Let $$A = X_{p+1:m,:}(X_{1:p,:}^tX_{1:p,:} + \alpha I)^{-1}X_{1:p,:}^tX_{1:p,:} - X_{p+1:m,:} $$$$= -\alpha X_{p+1:m,:}(X_{1:p,:}^tX_{1:p,:} + \alpha I)^{-1}$$ and $$ B =  X_{p+1:m,:}(X_{1:p,:}^tX_{1:p,:} + \alpha I)^{-1}X_{1:p,:}^t, $$ then this is
$$ \left( \frac{1}{m-p}\cdot\| Ab + \sigma B\epsilon_{1:p} - \sigma\epsilon_{p+1:m}\|^2 - \sigma^2 \right)^2 $$
which in E.V. is the Integrity Metric, and this becomes
$$ = \Bigg( \frac{1}{m-p}\Bigg( (\|Ab\|^2  - (m-p)\sigma^2) + \sigma^2\|B\epsilon_{1:p}\|^2 + \sigma^2\|\epsilon_{p+1:m}\|^2 $$$$ + 2\sigma b^tA^tB\epsilon_{1:p} - 2\sigma b^tA^t\epsilon_{p+1:m} - 2\sigma^2\epsilon_{1:p}^tB^t\epsilon_{p+1:m}\Bigg) \Bigg)^2 $$
$$ = \frac{1}{(m-p)^2}\Bigg( \left( (\|Ab\|^2  - (m-p)\sigma^2) + \sigma^2\|B\epsilon_{1:p}\|^2 + \sigma^2\|\epsilon_{p+1:m}\|^2\right)^2 $$$$+ 4\sigma^2 \left\|b^tA^tB\epsilon_{1:p}\right\|^2 + 4\sigma^2\left\|b^tA^t\epsilon_{p+1:m}\right\|^2 + 4\sigma^4(\epsilon_{1:p}^tB^t\epsilon_{p+1:m})^2\Bigg) $$
$$
 = \frac{1}{(m-p)^2}\Bigg(  (\|Ab\|^2  - (m-p)\sigma^2)^2 + \sigma^4\|B\epsilon_{1:p}\|^4 + \sigma^4\|\epsilon_{p+1:m}\|^4$$$$+2\sigma^2(\|Ab\|^2  - (m-p)\sigma^2)\|B\epsilon_{1:p}\|^2 + 2\sigma^2(\|Ab\|^2  - (m-p)\sigma^2)\|\epsilon_{p+1:m}\|^2 $$$$ + 2\sigma^4\|B\epsilon_{1:p}\|^2\|\epsilon_{p+1:m}\|^2   + 4\sigma^2 \left\|b^tA^tB\epsilon_{1:p}\right\|^2 + 4\sigma^2\left\|b^tA^t\epsilon_{p+1:m}\right\|^2 + 4\sigma^4(\epsilon_{1:p}^tB^t\epsilon_{p+1:m})^2\Bigg)
$$
which has expected value in common with
\begin{multline}
 \hspace*{-.2in} \frac{1}{(m-p)^2}\Bigg(  c^4\cdot{\rm tr}(A^tA)^2 + 2c^4\cdot{\rm tr}((A^tA)^2) +\sigma^4\cdot{\rm tr}(B^tB)^2 + 2\sigma^4\cdot{\rm tr}((B^tB)^2) + 2\sigma^4(m-p) \\ 
 \hspace*{.6in} + 2c^2\sigma^2\cdot{\rm tr}(A^tA){\rm tr}(B^tB)+4c^2\sigma^2\cdot{\rm tr}(B^tAA^tB) + 4c^2\sigma^2\cdot{\rm tr}(A^tA) + 4\sigma^4\cdot{\rm tr}(B^tB)\Bigg) \quad\end{multline}
\vskip .1 in
Replace $A$ and $B$ with $$A = -\alpha X_{p+1:m,:}G$$$$ B =  p^{1/2}X_{p+1:m,:}H, $$ \vskip .1 in where $G = \left(X_{1:p,:}^tX_{1:p,:} + \alpha I\right)^{-1}$ and $H = p^{1/2}\left(X_{1:p,:}^tX_{1:p,:} + \alpha I\right)^{-1}X_{1:p,:}^t$.
\vskip .05in
The expressions

$$ E({\rm tr}(A^tA)) = \alpha^2E\left({\rm tr}(G^tX_{p+1:m,:}^tX_{p+1:m,:}G)\right) $$
$$ E({\rm tr}(B^tB)) = p^{-1} E\left({\rm tr}(H^tX_{p+1:m,:}^tX_{p+1:m,:}H)\right) $$
$$ E({\rm tr}(A^tA)^{2}) = \alpha^4E\left({\rm tr}(G^tX_{p+1:m,:}^tX_{p+1:m,:}G)^2\right) $$
$$ E({\rm tr}(B^tB)^2) = p^{-2}\cdot E\left({\rm tr}(H^tX_{p+1:m,:}^tX_{p+1:m,:}H)^2\right) $$
$$ E({\rm tr}((A^tA)^2)) = \alpha^4E\left({\rm tr}((G^tX_{p+1:m,:}^tX_{p+1:m,:}G)^2)\right) $$
$$ E({\rm tr}((B^tB)^2)) = p^{-2}\cdot E\left({\rm tr}((H^tX_{p+1:m,:}^tX_{p+1:m,:}H)^2)\right) $$
$$ E({\rm tr}(B^tAA^tB)) = p^{-1}\alpha^2E\left({\rm tr}(H^tX_{p+1:m,:}^tX_{p+1:m,:}GG^tX_{p+1:m,:}^tX_{p+1:m,:}H)\right) $$
$$ E({\rm tr}(A^tA){\rm tr}(B^tB)) = p^{-1}\alpha^2E\left({\rm tr}(G^tX_{p+1:m,:}^tX_{p+1:m,:}G){\rm tr}(H^tX_{p+1:m,:}^tX_{p+1:m,:}H)\right) $$
\vskip .15 in
\noindent become by well-known properties of Gaussians,

\begin{equation*} E({\rm tr}(A^tA)) = \alpha^2E\left((m - p){\rm tr}(\Sigma GG^t)\right) \end{equation*}
\begin{equation*} E({\rm tr}(B^tB)) = p^{-1} E\left((m - p){\rm tr}(\Sigma HH^t)\right) \end{equation*}
\begin{equation*} E({\rm tr}(A^tA)^{2}) = \alpha^4E\left((m - p)^2{\rm tr}(\Sigma GG^t)^2 + 2(m-p){\rm tr}((\Sigma GG^t)^2)\right) \end{equation*}
\begin{equation*} E({\rm tr}(B^tB)^2) = p^{-2}E\left((m - p)^2{\rm tr}(\Sigma HH^t)^2 + 2(m-p){\rm tr}((\Sigma HH^t)^2)\right) \end{equation*}
\begin{equation*} E({\rm tr}((A^tA)^2)) = \alpha^4E\left((m - p){\rm tr}(\Sigma GG^t)^2 + (m - p)(m - p + 1){\rm tr}((\Sigma GG^t)^2)\right) \end{equation*}
\begin{equation*} E({\rm tr}((B^tB)^2)) = p^{-2}E\left((m - p){\rm tr}(\Sigma HH^t)^2 + (m - p)(m - p + 1){\rm tr}((\Sigma HH^t)^2)\right) \end{equation*}
$$\hspace*{-3.75in} E({\rm tr}(B^tAA^tB)) $$$$= p^{-1} \alpha^2E\left((m - p){\rm tr}(\Sigma GG^t){\rm tr}(\Sigma HH^t) + (m - p)(m - p + 1){\rm tr}(\Sigma GG^t\Sigma HH^t)\right) $$
$$\hspace*{-3.75in} E({\rm tr}(A^tA){\rm tr}(B^tB))$$$$ = p^{-1} \alpha^2E\left((m - p)^2{\rm tr}(\Sigma GG^t){\rm tr}(\Sigma HH^t) + 2(m-p){\rm tr}(\Sigma GG^t\Sigma HH^t)\right) $$

Now expression (13) becomes expression (14), and the above eight expected values can be evaluated using Lemma 5 and plugged into (14), so (13) is a rational function in $m$ and $p$ (plus some big-$O$ terms) which we would like to minimize using calculus. We find it by using Mathematica and plugging in, and we have a polynomial in $m$ and $p$ for the Integrity Metric, which we call call it $f(m,p)$. Let $g(m,p)$ be the polynomial obtained by finding $\frac{\partial f}{\partial p}$.

\begin{multline*}
-(m-p)^2p^5 g(m, p) = 4 m^2 n p^2 \sigma^4 + 2 m^2 n^2 p^2 \sigma^4 
- 4 m n p^3 \sigma^4 - 4 m n^2 p^3 \sigma^4 \\
- 4 n p^4 \sigma^4 + 2 n^2 p^4 \sigma^4 
- 2 p^5 \sigma^4 + \text{L.O.T.}
\end{multline*}

We want to set this expression to zero to find $p$ as a function of $m$ in the limit as $m$ becomes large. All of the L.O.T. are of strictly lower order if any terms of $p$ as a function of $m$ are of order at least $m^{1/3}$. We can solve it using Mathematica to find $p^*(m)$ and take the limit as $m\rightarrow\infty$ of $p^*(m)/m^{2/3}$, which is $n^{1/3}(2 + n)^{1/3}$, and then find the limit as $m\rightarrow\infty$ of $\left(p^*(m) - n^{1/3}(2 + n)^{1/3}\right)/m^{1/3}$ to get $- \frac{2n^{2/3}(1 + n)}{3(2 + n)^{1/3}}$, the desired result.

\end{proof}

\section{Computational Evidence}

For small $m$ and $n$, it is possible to find $p^*(m)$ computationally and compare it to the first two terms of the formula given in Theorem 6. For each panel in Figures 1 and 2, $n$, $c$, $\sigma$, and $\alpha$ are chosen (see the caption), $\Sigma$ is one positive definite matrix generated randomly, and the following is done $10^5$ times: $b$, $\epsilon$, and $X$ are sampled randomly and the Integrity Metric is approximated via the Law of Large Numbers for every value of $p$ ranging from $n$ to $m$. The minimizing $p$ is chosen, and it is taken to be $p^*(m)$, plotted in red. Note that convergence is slow since taking an $\arg\min$ of a random quantity is inaccurate.

\section{Discussion}

We have shown that the author's previous result in \cite{Dubbs2024}, that training set sizes are $O(m^{2/3})$ where $m$ is the number of data points, is resilient to the addition of ridge regularization, in the first result of its kind. It seems very likely that this is a general rule for linear models, but tree-based models and deep learning models may work differently. Future research should establish whether this phenomenon holds for these nonlinear models.

\begin{figure}
    \centering
    \begin{minipage}{0.48\textwidth}
        \centering
        \includegraphics[width=\linewidth]{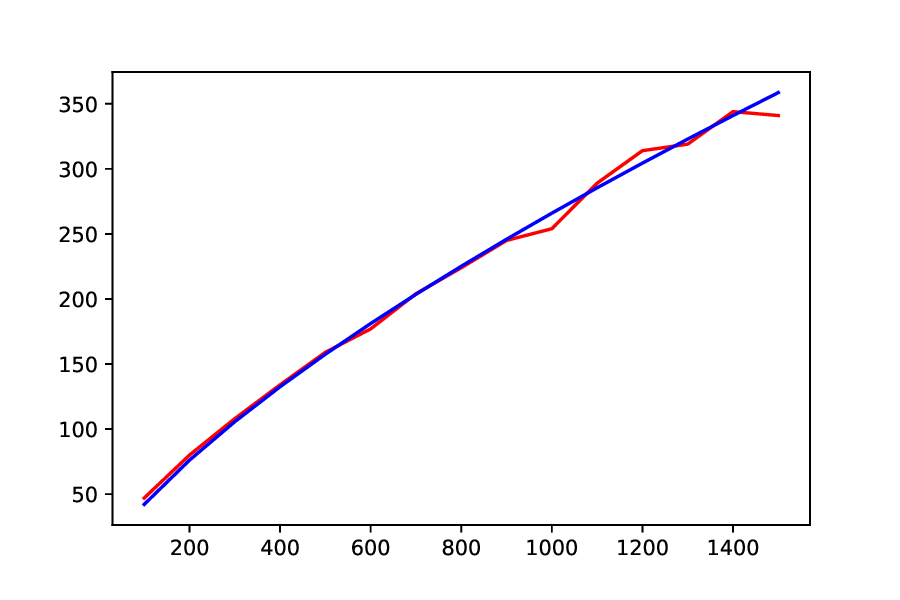}
        \subcaption{Panel 1.}
    \end{minipage}
    \hfill
    \begin{minipage}{0.48\textwidth}
        \centering
        \includegraphics[width=\linewidth]{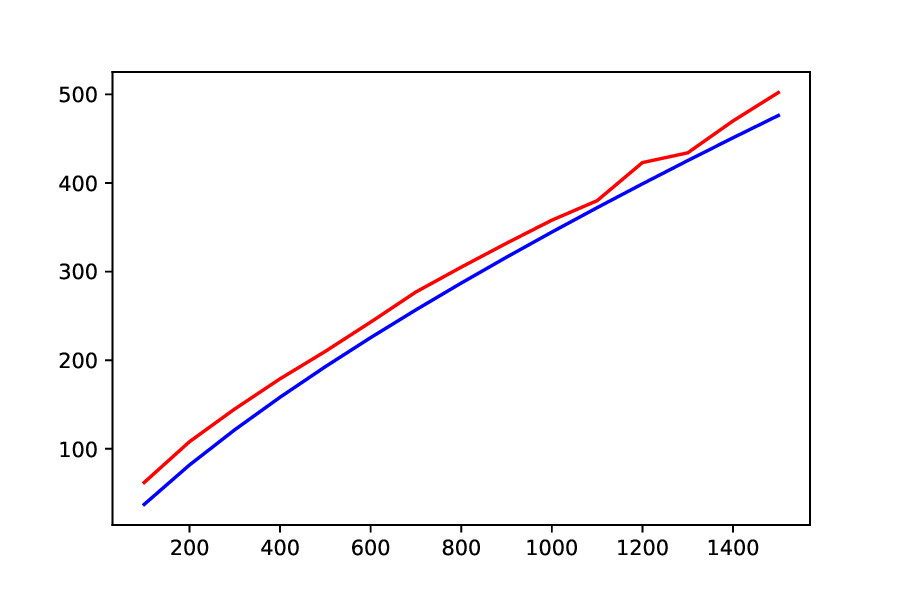}
        \subcaption{Panel 2.}
    \end{minipage}
    \vspace{0.5cm}
    \begin{minipage}{0.48\textwidth}
        \centering
        \includegraphics[width=\linewidth]{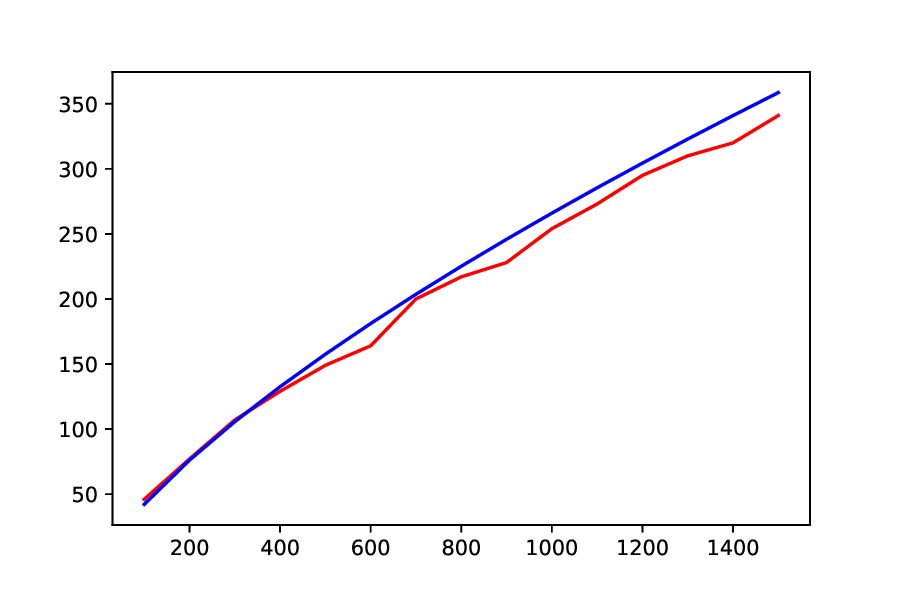}
        \subcaption{Panel 3.}
    \end{minipage}
    \hfill
    \begin{minipage}{0.48\textwidth}
        \centering
        \includegraphics[width=\linewidth]{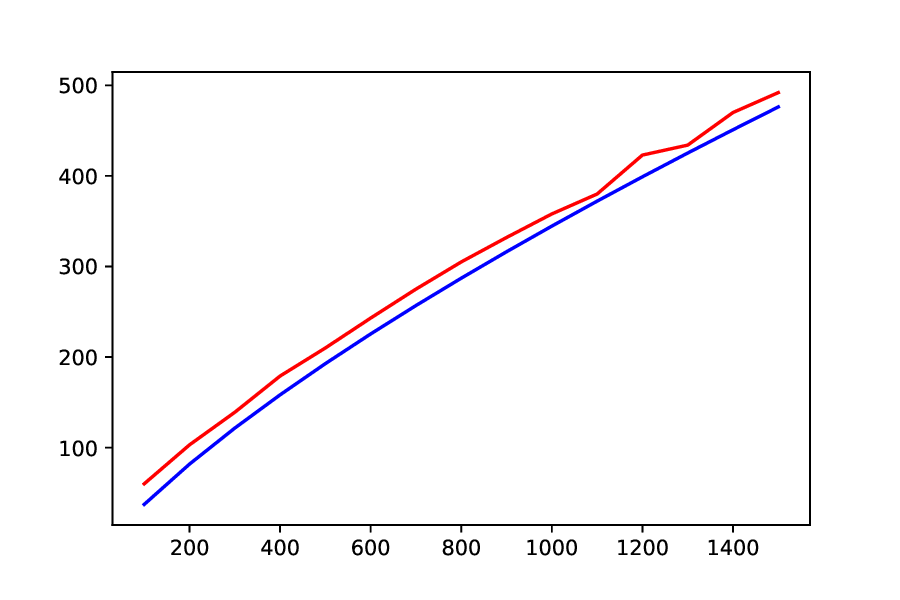}
        \subcaption{Panel 4.}
    \end{minipage}
    \caption{Odd numbered panels have $n = 5$, even numbered panels have $n = 10$, all have $c = 0.1$. The first two panels have $\sigma = 0.1$, the second two have $\sigma = 0.2$. All have $\alpha = 2$. Analytic computation in blue, numerical in red.}
    \label{fig:fourpanel}
\end{figure}

\begin{figure}
    \centering
    \begin{minipage}{0.48\textwidth}
        \centering
        \includegraphics[width=\linewidth]{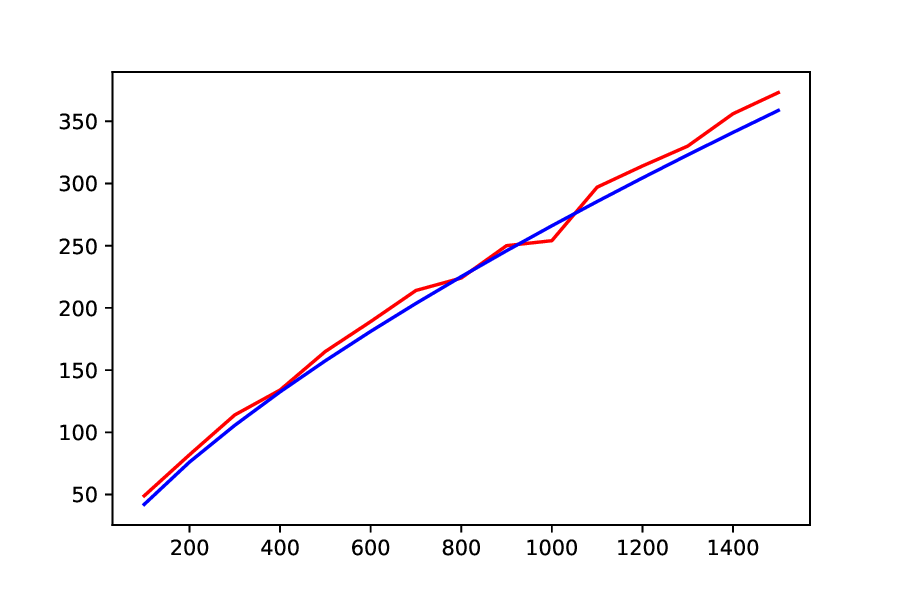}
        \subcaption{Panel 5.}
    \end{minipage}
    \hfill
    \begin{minipage}{0.48\textwidth}
        \centering
        \includegraphics[width=\linewidth]{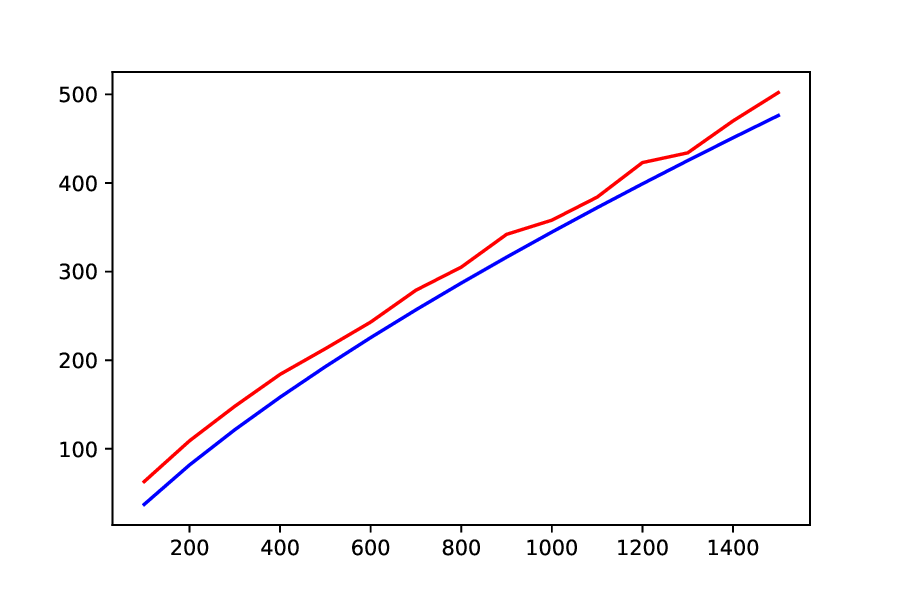}
        \subcaption{Panel 6.}
    \end{minipage}
    \vspace{0.5cm}
    \begin{minipage}{0.48\textwidth}
        \centering
        \includegraphics[width=\linewidth]{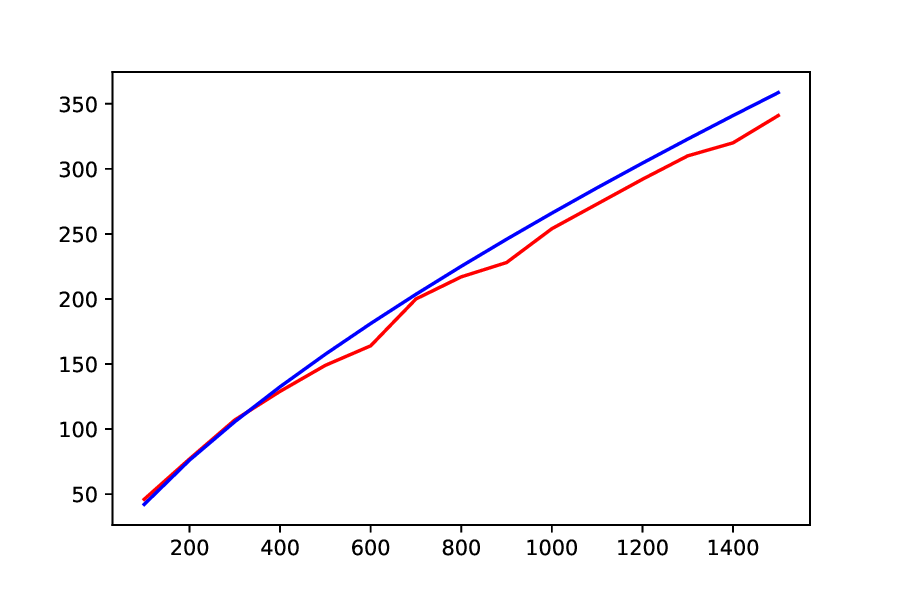}
        \subcaption{Panel 7.}
    \end{minipage}
    \hfill
    \begin{minipage}{0.48\textwidth}
        \centering
        \includegraphics[width=\linewidth]{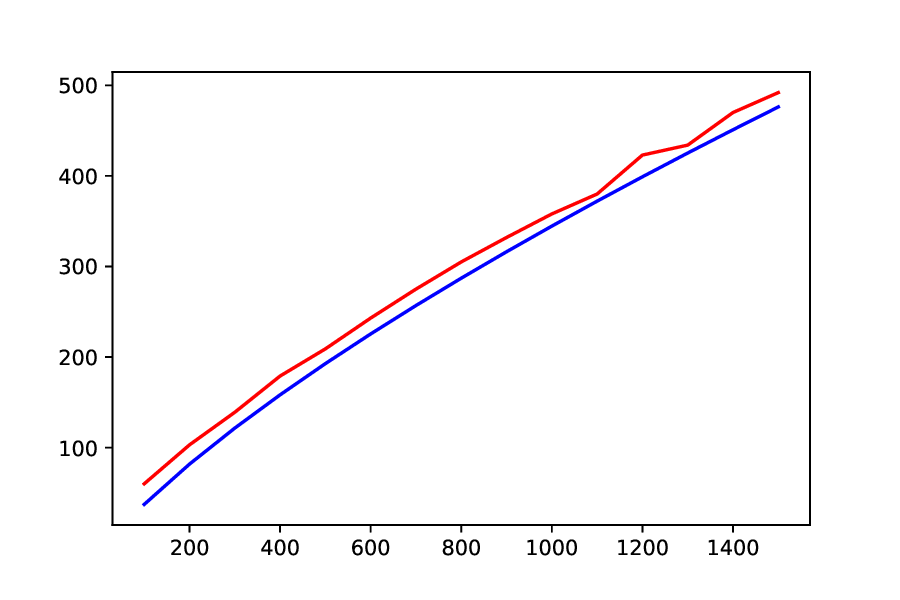}
        \subcaption{Panel 8.}
    \end{minipage}
    \caption{Odd numbered panels have $n = 5$, even numbered panels have $n = 10$, all have $c = 0.1$. The first two panels have $\sigma = 0.1$, the second two have $\sigma = 0.2$. All have $\alpha = 4$.  Analytic computation in blue, numerical in red.}
    \label{fig:fourpanel}
\end{figure}

\section*{Declarations}

\subsection*{Funding}

Not applicable.

\subsection*{Conflicts of Interest}

None.

\subsection*{Ethics Approval}

Not applicable.

\subsection*{Consent to Participate}

Not applicable.

\subsection*{Consent for Publication}

Not applicable.

\subsection*{Availability of Data and Materials}

Not applicable.

\subsection*{Code Availability}

Code is available upon request.

\subsection*{Authors' Contributions}

This article is solely the work of its one author.
\section*{Declarations}

\subsection*{Funding}

Not applicable.

\subsection*{Conflicts of Interest}

None.

\subsection*{Ethics Approval}

Not applicable.

\subsection*{Consent to Participate}

Not applicable.

\subsection*{Consent for Publication}

Not applicable.

\subsection*{Availability of Data and Materials}

Not applicable.

\subsection*{Code Availability}

Code is available upon request.

\subsection*{Authors' Contributions}

This article is solely the work of its one author.

\end{document}